\documentclass[]{article}
\usepackage{proceed2e}

\usepackage{algorithmicx}
\usepackage[noend]{algpseudocode}
\usepackage{listings}
\usepackage{caption}
\usepackage{subcaption}

\usepackage{times}

\newcommand{\cG}{{\cal G}}
\newcommand{\cM}{{\cal M}}
\newcommand{\cO}{{\cal O}}

\newcommand{\cU}{{\cal U}}
\newcommand{\cB}{{\cal B}}

\newcommand{\cS}{{\cal S}}

\newcommand{\cP}{{\cal P}}

\newcommand{\bC}{{C}}
\newcommand{\bQ}{{Q}}
\newcommand{\bW}{{W}}
\newcommand{\bV}{{V}}
\newcommand{\bE}{{E}}
\newcommand{\bI}{{I}}

\newcommand{\bR}{{R}}

\newcommand{\bX}{{X}}
\newcommand{\bY}{{Y}}
\newcommand{\bZ}{{Z}}

\newcommand{\sdgsym}{\ensuremath{\cU}}
\newcommand{\sdgabbrv}{SMIG}
\newcommand{\sdgset}{{\bf SMIG}}

\usepackage{pdfpages}
\usepackage{tikz}
\usetikzlibrary{arrows}

\usepackage{amsthm}
\usepackage{amsmath}
\usepackage{amssymb}
\usepackage[authoryear,round]{natbib}

\usepackage{caption}
\usepackage{subcaption}
\usepackage{graphicx}

\usepackage{algorithmicx}
\usepackage[noend]{algpseudocode}

\DeclareMathAlphabet{\mathpzc}{OT1}{pzc}{m}{it}
 
\newtheorem{theorem}{Theorem}[section]
\newtheorem{corollary}[theorem]{Corollary}
\newtheorem{lemma}[theorem]{Lemma}
\newtheorem{definition}[theorem]{Definition} 
\newtheorem{proposition}[theorem]{Proposition}

\title{Learning from Pairwise Marginal Independencies}

\author{
{\bf Johannes Textor}\\  
Theoretical Biology \& Bioinformatics\\  
Utrecht University, The Netherlands \\
johannes.textor@gmx.de
\And
{\bf Alexander Idelberger} \\  
Theoretical Computer Science\\  
University of L\"{u}beck, Germany \\
alex@pirx.de
\And
{Maciej Li\'{s}kiewicz} \\  
Theoretical Computer Science\\  
University of L\"{u}beck, Germany \\
liskiewi@tcs.uni-luebeck.de
}

\tikzset{bi/.style={-,draw,shorten <=1.5pt, shorten >=1.5pt}}
\tikzset{dir/.style={->,draw,shorten <=0pt, shorten >=1.5pt}}
\tikzset{rdir/.style={<-,draw,shorten >=0pt, shorten <=1.5pt}}

\begin{document}

\maketitle

\begin{abstract}
We consider graphs that represent pairwise marginal independencies amongst a set of variables (for instance, the zero entries of a covariance matrix for normal data). We characterize the directed acyclic graphs (DAGs) that faithfully explain a given set of independencies, and derive algorithms to efficiently enumerate such structures. Our results map out the space of faithful causal models for a given set of pairwise marginal  independence relations. This allows us to show the extent to which causal inference is possible without using conditional independence tests.
\end{abstract}

\section{INTRODUCTION}

DAGs and other graphical models encode 
conditional independence (CI) relationships in
probability distributions. Therefore, CI tests 
are a natural building block of algorithms that infer such models from
data. For example, the PC algorithm for learning DAGs \citep{Kalisch2007} and the FCI
\citep{Spirtes2000} and RFCI \citep{Colombo2012} algorithms for learning 
maximal ancestral graphs are all based on CI tests. 

CI testing is 
still an ongoing research topic, to which the 
UAI community is contributing 
\citep[e.g.][]{Zhang2011,Doran2014}. But
at least for continuous variables,
CI testing will always remain more difficult than 
testing marginal independence for quite fundamental
reasons \citep{Bergsma2004}. Intuitively, the difficulty is 
that two variables $x$ and $y$
could be dependent ``almost nowhere'',
e.g., for only a few values of the conditioning
variable $z$. 
This suggests a two-staged approach to structure learning:
first try to learn as much as possible from simpler 
independence tests before applying 
CI tests. Here, we present a theoretical
basis for extracting as much information 
as possible from 
the simplest kind of stochastic independence --
pairwise marginal independence. 

More precisely, we will consider the following problem.
We are given the set of pairwise
marginal independencies that hold amongst some
variables of interest. Such sets can be represented
as graphs whose missing edges correspond to independencies
(Figure~\ref{fig:marginalbutnotmarkov}a). We call
such graphs \emph{marginal independence graphs}.
We wish to find DAGs on the same variables that
entail exactly the given set of pairwise marginal independencies 
(Figure~\ref{fig:marginalbutnotmarkov}b). We call 
such DAGs \emph{faithful}. 
Sometimes no such DAGs exist (e.g., Figure~\ref{fig:marginalbutnotmarkov}c).
Else, we are interested in finding the set of \emph{all}
faithful DAGs, hoping that this set will be substantially
smaller than the set of all possible DAGs on the same variables.
Those candidate DAGs could then be probed further by using joint
marginal or conditional independence tests.

\begin{figure}
\null\hfill
\begin{tikzpicture}[xscale=2]
\tikzstyle{every node}=[circle,fill,inner sep=1pt];
\tikzstyle{every edge}=[bi];
\node [fill=black] (a1) at (-0.5,0) {};
\node [fill=black] (a2) at (0,0) {};
\node  (a3) at (0.5,0) {};
\node (b1) at (-0.25,0.5) {};
\node [fill=black] (b2) at (0.25,0.5) {};
\node [fill=black]   (c1) at (0,1) {};
\draw  (b1) edge (c1) edge (b2) edge (a1) edge (a2)
      (a2) edge (b2) edge (a3) edge (a1)
     (b2) edge (c1) edge (a3);
\node [fill=none] at (0,-.5) {(a)};
\end{tikzpicture}
\hfill
\begin{tikzpicture}[xscale=2]
\tikzstyle{every node}=[circle,fill,inner sep=1pt];
\tikzstyle{every edge}=[dir];
\node [fill=black] (a1) at (-0.5,0) {};
\node [fill=black] (a2) at (0,0) {};
\node  (a3) at (0.5,0) {};
\node (b1) at (-0.25,0.5) {};
\node [fill=black] (b2) at (0.25,0.5) {};
\node [fill=black]   (c1) at (0,1) {};
\draw (a1) edge (b1) edge (a2);
\draw (a3) edge (a2) edge (b2);
\draw (c1) edge (b2) edge (b1);
\node [fill=none] at (0,-.5) {(b)};
\end{tikzpicture}
\hfill
\begin{tikzpicture}[xscale=2]
\tikzstyle{every node}=[circle,fill,inner sep=1pt];
\tikzstyle{every edge}=[bi];
\node [fill=black] (a1) at (-0.5,0) {};
\node [fill=black] (a2) at (0,0) {};
\node (b1) at (-0.25,0.5) {};
\node [fill=black] (b2) at (0.25,0.5) {};
\node [fill=black]   (c1) at (0,1) {};
\draw (b2) edge (c1) 
(a2) edge (b1) edge (b2) 
(a1) edge (a2) edge (b1) 
(b1) edge (b2) edge (c1);
\node [fill=none] at (0,-.5) {(c)};
\end{tikzpicture}
\hfill\null
\caption{(a) A \emph{marginal independence graph}
 $\sdgsym$ whose missing edges 
represent pairwise marginal 
independencies. (b) A \emph{faithful DAG} $\cG$ entailing 
the same set of pairwise marginal independencies
as $\sdgsym$. (c) A graph for which no such faithful DAG exists.}
\label{fig:marginalbutnotmarkov}
\end{figure}
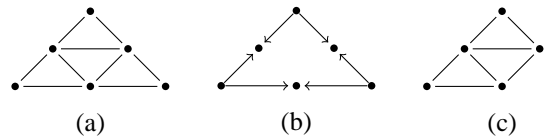

Other authors have represented marginal
(in)dependencies using bidirected graphs 
\citep{Drton2002,Richardson2003,Drton2008b}, instead of 
undirected graphs like we do here. 
We hope that the reader is compensated for this small departure from
community standards by the lower amount of clutter in our figures,
and the greater ease to link our work to standard graph 
theoretical results.
We also emphasize that we model only pairwise, and not 
higher-order joint dependencies.
However, for Gaussian data, pairwise independence entails
joint independence. In that case, our marginal 
independence graphs are equivalent to \emph{covariance graphs}
\citep{Cox1993,Pearl1994,Drton2002,Drton2008,Penha2013},
whose missing edges represent zero covariances.

Our results generalize the work of \citet{Pearl1994}
who showed (but did not prove) how to find \emph{some} faithful DAGs 
for a given covariance graph. We review
these and other connections to related work in
Section~\ref{sec:concepts} where we also 
link our problem to the theory of partially
ordered sets (posets). This connection allows us
to identify certain maximal and minimal faithful 
DAGs. Based on these ``boundary
DAGs'' we then derive 
a characterization of 
all faithful DAGs (Section~\ref{sec:consistency}),
and construct related enumeration algorithms
(Section~\ref{sec:algo}).
We use these algorithms
to explore the combinatorial structure of
faithful DAG models
(Section~\ref{sec:combinatorics}) which leads,
among other things, to a quantification of how much
pairwise marginal independencies reduce structural
causal uncertainty.
Finally, we ask what happens when a set
of independencies can \emph{not} be explained by any DAG.
How many additional variables will we
need? We prove that 
this problem is NP-hard (Section~\ref{sec:latents}).

Preliminary versions of many of the results presented 
in this paper were obtained
in the Master's thesis of the second author \citep{idelberger2014}.

\section{PRELIMINARIES} 

In this paper we use the abbreviation \emph{iff}
for the  connective ``if and only if''. 
A graph $\cG=(V,E)$ consists of a set of nodes
(variables) $V$ and set of edges $E$.
We consider 
undirected graphs (which we simply refer to as graphs),
directed graphs, and mixed graphs 
that can have both undirected edges (denotes as $x - y$) and directed
edges (denoted as $x \to y$). 
Two nodes are \emph{adjacent} if they
are linked by any edge. 
A \emph{clique} in a graph is 
a node set $\bC \subseteq \bV$ such that all 
$u,v \in \bC$ are adjacent. 
Conversely, an \emph{independent set}
is a node set  $\bI \subseteq \bV$  in which no two
 nodes $u,v\in\bI$ are adjacent.
A \emph{maximal clique}
is a clique for which no proper superset of nodes
is also a clique. For any $v \in V$, the
 \emph{neighborhood} $N(v)$ is the
set of nodes adjacent to  $v$ and the \emph{boundary}
$\textit{Bd}(v)$ is the neighborhood of $v$ including $v$, 
i.e. $\textit{Bd}(v)=N(v)\cup \{v\}$.
A node $v$ is called \emph{simplicial} if $\textit{Bd}(v)$ is
a clique. Equivalently, $v$ is simplicial iff 
$\textrm{Bd}(v) \subseteq \textrm{Bd}(w)$ for all $w \in N(v)$
\citep{kloks00}.
A clique that contains simplicial nodes
is called a \emph{simplex}.
Every simplex is a maximal
clique, and every simplicial node belongs
to exactly one simplex. The \emph{degree} $d(v)$
of a node $v$ is $|N(v)|$. If for two graphs
$\cG=(\bV,\bE(\cG))$ and $\cG'=(\bV,\bE(\cG'))$ we have
$\bE(\cG)\subseteq\bE(\cG')$, then 
$\cG$ is an \emph{edge subgraph} of $\cG'$ 
and $\cG'$ is an \emph{edge supergraph} of $\cG$.
The \emph{skeleton}
of a directed graph $\cG$ is obtained by 
replacing every edge $u \to v$ by an undirected
edge $u - v$.

A \emph{path} of length $n-1$ is a sequence of $n$ distinct 
nodes in which successive nodes are pairwise adjacent. 
A \emph{directed path}  $x \to \ldots \to y$ consists of directed edges
that all point towards $y$. 
In a directed graph, a node $u$ is an 
\emph{ancestor} of another node $v$ if $u=v$ or 
if there is a directed path $u \to \cdots \to v$. For each edge
$u \to v$, we say that $u$ is a \emph{parent} of $v$
and $v$ is a \emph{child} of $u$.
If two nodes $u,v$ in a directed graph 
have a common ancestor $w$ (which can be $u$ or $v$),
then the path
$u \gets \ldots \gets w \to \ldots \to v$
is called a \emph{trek} connecting $u$ and $v$.
A DAG is called \emph{transitive} if, for all
$u \neq v$,  it contains an edge $u \to v$ whenever
there is a directed path from $u$ to $v$. Given a DAG
$\cG$, the \emph{transitive closure} is the unique
transitive graph that implies the same ancestor
relationships as $\cG$, whereas the \emph{transitive
reduction} is the unique edge-minimal graph 
that implies the same ancestor
relationships.

In this paper we encounter several well-known graph
classes, e.g., chordal graphs and 
trivially perfect graphs. We will give brief definitions
when appropriate, but we direct the reader to
the excellent survey by \citet{Brandstaedt1999} for
 further details.

\section{SIMPLE MARGINAL INDEPENDENCE GRAPHS}

\label{sec:concepts}

In this section we define the class of  graphs
which can be explained using a directed acyclic 
graph (DAG) on the same variables.
We will refer to such graphs as \emph{simple marginal independence
graphs} (\sdgabbrv{}s).

\begin{definition}\label{definition:dependency:graph}
A graph $\sdgsym=(\bV,E(\sdgsym))$ 
is called the \emph{simple marginal independence graph} (\sdgabbrv), 
or \emph{marginal independence graph} of 
a DAG $\cG=(\bV,E(\cG))$ if for all
$v,w\in \bV$, $v-w \in E(\sdgsym)$ iff
$v$ and $w$ have a common ancestor in~$\cG$.
If $\sdgsym$ is the marginal independence graph of $\cG$
then we also say that $\cG$ is 
\emph{faithful} to $\sdgsym$.
$\sdgset$ is the set of all graphs 
$\sdgsym$ for which there exists a 
faithful DAG $\cal G$. Note that each DAG has
exactly one marginal independence graph.
\end{definition}

Again, we point out that marginal independence graphs
are often called (and drawn as) \emph{bidirected graphs} 
in the literature, though the term ``marginal independence graph''
has also been used by various authors
\citep[e.g.][]{Tan2014}.

\subsection{\sdgabbrv{}s and Dependency Models}

In this subsection we recall briefly the general
setting for modeling (in)dependencies proposed by 
\cite{pearl1987logic} and show the relationship 
between that model and \sdgabbrv{}s.
In the definitions below $\bV$ 
denotes a set of variables
and $\bX$, $\bY$ and $\bZ$ are 
three disjoint subsets of $\bV$.

\begin{definition}[\cite{pearl1987logic}]
A \emph{dependency model} $\cM$ over $\bV$ is
any subset of triplets $(\bX , \bZ, \bY)$ which 
represent independencies,
that is, $(\bX, \bZ, \bY) \in \cM$ asserts that 
$\bX$ is independent of $\bY$  given $\bZ$.

%

A \emph{probabilistic dependency model} $\cM_P$ is
defined in terms of a probability distribution $P$ over~$\bV$. 
By definition $(\bX, \bZ, \bY) \in \cM_P$ iff 
for any instantiation $\hat{x}$, $\hat{y}$ and $\hat{z}$ of the variables in these
subsets $P(\hat{x} \mid \hat{y} \ \hat{z} ) = P(\hat{x} \mid \hat{z})$.

%

A \emph{directed acyclic graph dependency model}
$\cM_{\cG}$ is defined in terms of a DAG $\cG$. 
By definition $(\bX, \bZ, \bY)\in \cM_{\cG}$
iff $\bX$ and $\bY$ are $d$-separated by $\bZ$ in~$\cG$ 
(for a definition of $d$-separation by a set $\bZ$ see \cite{pearl1987logic}).
\end{definition}

We define a \emph{marginal} dependency model, resp. 
marginal probabilistic and marginal DAG dependency model, 
analogously as \cite{pearl1987logic} with the restriction 
that the second component of any triple $(\bX, \bZ, \bY)$ 
is the empty set. 
Thus, such marginal dependency models are
sets of pairs $(\bX, \bY)$.
It is easy to see that the following 
properties are satisfied.

\begin{lemma}
Let $\cM$ be a marginal probabilistic dependency model
or a marginal DAG dependency model. 
Then $\cM$ is closed under:\\ 
\hspace*{3mm}Symmetry: $(\bX, \bY)\in \cM \  \Leftrightarrow\ (\bY, \bX)\in \cM$ and \\
\hspace*{3mm}Decomposition: $(\bX, \bY\cup \bW)\in \cM \  \Rightarrow\ (\bX, \bY)\in \cM$.\\
Moreover, if $\cM$ is a marginal DAG dependency model then it is also closed under\\
\hspace*{3mm}Union: $(\bX, \bY), (\bX, \bW)\in \cM \  \Rightarrow\ (\bX, \bY\cup \bW)\in \cM$.
\end{lemma}

The marginal probabilistic dependency model 
is not closed under union in general. 
For instance, consider 
two independent, uniformly distributed
 binary variables $y$ and $w$
 and let $x=y\oplus w$, 
where $\oplus$ denotes xor of two bits.
For the model 
$\cM_P$ defined in terms of probability 
over $x,y,w$ we have that $(\{x\},\{y\})$ and 
$(\{x\},\{w\})$ belong to $\cM_P$ but 
$(\{x\},\{y,w\})$ does not.

In this paper we will \emph{not} assume that the marginal
independencies in the data are closed under union.
Instead, we only consider pairwise independencies,
which we formalize as follows. 

\begin{definition}
Let $\cM$ be a marginal probabilistic dependency model
over $V$.
Then the simple marginal independence graph $\cU=(V,E(\sdgsym))$ 
of $\cM$ is the graph in which $x-y \in E(\sdgsym)$ 
iff $(\{x\},\{y\})\not\in \cM$. 
\end{definition}

Thus, in general, marginal independence graphs do not contain any information
on higher-order \emph{joint} independencies present
in the data. However, under certain common parametric assumptions, 
dependency models would be closed under union as well. This 
holds, for instance, if the data are normally distributed. In that case,
marginal independence is equivalent to zero covariance, pairwise independence
implies joint independence, and marginal independence graphs become covariance graphs.


%

The following is not difficult to see.

\begin{proposition}
A marginal dependency model $\cM$ which is closed under 
symmetry, decomposition, and union coincides with 
the transitive closure of 
$\{(\{x\},\{y\}): x,y \in \bV\}\cap \cM$
over symmetry and union.
\end{proposition}

This Proposition entails that if the marginal dependencies in the data are closed
under these properties, then the entire marginal dependency
model is represented by the marginal 
independence graph.




\subsection{\sdgabbrv{}s and Partially Ordered Sets}

To reach our aim of a complete and constructive
characterization of the DAGs faithful to a given \sdgabbrv,
it is useful to observe that marginal independence graphs are invariant
with respect to the insertion or deletion of transitive edges from 
the DAG. We formalize this as follows.

\begin{definition}
A (labelled) \emph{poset} $\cP$ is a DAG that is
identical to its transitive closure.
\end{definition}

\begin{proposition} The marginal independence graphs
 of a DAG $\cG$ and its transitive closure
 $\cP(\cG)$ are identical.
 \end{proposition}
 
 \begin{proof}
 Two nodes are not adjacent in the marginal independence graph
iff they have no common ancestor in the DAG. 
Transitive edges do not influence ancestral relationships.
 \end{proof}

We thus restrict our attention to finding \emph{posets} 
that are faithful to a given \sdgabbrv. Note that faithful DAGs can 
then be obtained by deleting transitive edges from faithful posets;
since no DAG obtained in this way can be an edge subgraph of two different posets,
this construction is unique and well-defined. In particular, by 
deleting \emph{all} transitive edges from a poset, we obtain 
a sparse graphical representation of the poset as defined below. 

\begin{definition}
Given a poset $\cP=(\bV,\bE)$, its \emph{transitive reduction} is the unique
DAG $\cG_\cP=(\bV,\bE')$ for which $\cP(\cG)=\cP$ and 
 $\bE'$ is the smallest set where 
$\bE' \subseteq \bE$.
\end{definition}

Transitive reductions are also known as \emph{Hasse diagrams}, 
though Hasse diagrams are usually unlabeled. 
Different posets can have the same marginal independence
graphs, e.g. the posets with Hasse diagrams
$\cP_1= x \to y \to z$ and $\cP_2= x \gets y \to z$.
Similarly, Markov equivalence is a sufficient but not 
necessary condition to inducing the same marginal independence 
graphs (adding an edge $x\to z$ to $\cP_2$ changes 
the poset and the Markov equivalence class, but 
not the marginal independence graph). 

\subsection{Recognizing \sdgabbrv{}s}

\label{sec:character}

We first recall existing results that show
which graphs admit a faithful DAG at all,
and how to find such DAGs if possible.
Note that many of these
results have been stated without
proof \citep{Pearl1994}, but our connection to
posets will make some of these proofs straightforward.
The following notion related to posets is required.

\begin{definition}[Bound graph \citep{Morris1982}]
For a poset $\cP=(\bV,\bE)$, the \emph{bound graph} 
$\cB=(\bV,\bE')$ of $\cP$
is the graph where
$x - y \in \bE'$ iff $x$ and $y$ share a \emph{lower bound}, 
i.e., have a common ancestor in $\cP$.
\label{definition:boundgraph}
\end{definition}

\begin{theorem}
\label{thm:graphcharact}
$\sdgset$ is the set of all graphs
for which every edge is contained in a simplex.
\end{theorem}

\begin{proof}
This is Theorem~2 in \citet{Pearl1994}
(who referred to simplexes as ``exterior cliques'').
Alternatively, we can 
observe that the marginal independence graph
$\sdgsym$ of a poset $\cP$ (Definition~\ref{definition:dependency:graph}) 
is equal to its  bound graph (Definition~\ref{definition:boundgraph}). 
The characterization of 
bound graphs as ``edge simplicial'' graphs
has been proven by \citet{Morris1982} by noting
that simplicial nodes in $\sdgsym$ correspond to 
possible minimal elements in $\cP$.
We note that this result
predates the equivalent 
statement in \citet{Pearl1994}.
\end{proof}

Though all bound graphs have a faithful poset,
not all bound graphs have one with
the same skeleton; see Figure~\ref{fig:marginalbutnotmarkov}a,b
for a counterexample. However, the graphs for
which a poset with the same skeleton can be found
are nicely characterizable in terms of forbidden
subgraphs.

\begin{theorem}[\citet{Pearl1994}]
Given a graph $\sdgsym$, a DAG $\cG$ that is
faithful to $\sdgsym$ and has the same skeleton exists iff
$\sdgsym$ is trivially perfect (i.e., $\sdgsym$
has no $P_4$=\tikz[scale=.3,shorten <=1pt,shorten >=1pt,baseline=-2pt] 
\draw (1,0) edge [-] (0,0) edge  [-] (2,0)
(3,0) edge  [-] (2,0);
nor a $C_4$=
\tikz[scale=.3,shorten <=1pt,shorten >=1pt,baseline=1.8pt] \draw (0,0) edge [-] (1,0) edge  [-] (0,1)
(1,1) edge  [-] (0,1) edge [-] (1,0);
\ 
as 
induced subgraph).
\label{thm:triviallyperfect}
\end{theorem}

It is known that the trivially perfect graphs 
are the intersection of the bound graphs and the 
chordal graphs \citep[Figure~\ref{fig:inclusions}; ][]{Grant2006}.

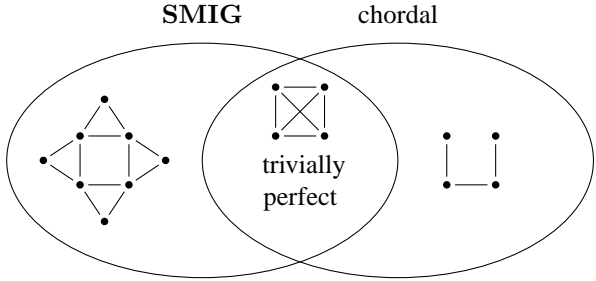
\begin{figure}

\centering

\begin{tikzpicture}[scale=1.3]
\node (cog) at (1,1.5) {chordal};
\node (mdg) at (-1,1.5) {$\sdgset$};
\node [text width=1cm,align=center] (tpg)  at (0,0) {\ \\ \ \\ trivially\\perfect};

\tikzstyle{every edge}=[bi]
\tikzstyle{every node}=[circle,fill=black,inner sep=1pt];

\begin{scope}[yshift=0.5cm,scale=0.5]
\node (a) at (-0.5,-0.5) {};
\node (b) at (0.5,-0.5) {};
\node (c) at (-0.5,0.5) {};
\node (d) at (0.5,0.5) {};
\draw (a) edge (b);
\draw (a) edge (c);
\draw (b) edge (d);
\draw (c) edge (d);
\draw (a) edge (d);
\draw (b) edge (c);
\end{scope}

\draw (-1,0) ellipse (2cm and 1.2cm);
\draw (1,0) ellipse (2cm and 1.2cm);

\begin{scope}[xshift=1.75cm,scale=0.5]
\node (a) at (-0.5,-0.5) {};
\node (b) at (0.5,-0.5) {};
\node (c) at (-0.5,0.5) {};
\node (d) at (0.5,0.5) {};
\draw (a) edge (b);
\draw (a) edge (c);
\draw (b) edge (d);
\end{scope}


\begin{scope}[xshift=-2cm,yshift=0cm,scale=0.5]
\node (a) at (-0.5,-0.5) {};
\node (b) at (0.5,-0.5) {};
\node (c) at (-0.5,0.5) {};
\node (d) at (0.5,0.5) {};
\node (e) at (1.25,0) {};
\node (f) at (-1.25,0) {};
\node (g) at (0,-1.25) {};
\node (h) at (0,1.25) {};
\draw (b) edge (a) edge (d) (c) edge (a) edge (d);
\draw (e) edge (b) edge (d);
\draw (f) edge (a) edge (c);
\draw (g) edge (a) edge (b);
\draw (h) edge (c) edge (d);
\end{scope}



\end{tikzpicture}

\caption{Relation between chordal
graphs, trivially perfect graphs, 
and $\sdgset$. In
 graph theory, $\sdgset$ is known as
the class of (upper/lower) \emph{bound graphs} \citep{Grant2006}.}
\label{fig:inclusions}
\end{figure}

This nice result begs the question whether a similar
characterization is also possible for $\sdgset$. As the following 
observation shows, that is not the case.

\begin{proposition}
Every graph $\sdgsym$ is an induced subgraph of some 
graph $\sdgsym' \in \sdgset$.
\label{proposition:induced:subgraph:MDG}
\end{proposition}

\begin{proof}
Take any graph ${\cal U}=(V,E)$ and construct a new graph
${\cal U}'$ as follows. For every edge $e=u-v$ in $\cal U$, add a new
node $v_e$ to $V$ and add edges $v_e-u$ and  $v_e-v$. Obviously
${\cal U}$ is an induced subgraph of ${\cal U}'$. To see 
that ${\cal U}'$ is in $\sdgset$, consider the
DAG $\cal G$ consisting of the nodes in ${\cal U}'$ and the edges $v \gets v_e \to u$ and 
for each newly added node in ${\cal U}'$. Then $\cal U$ is the marginal 
independence graph
of $\cal G$.
\end{proof}

The graph class characterization implies efficient
recognition algorithms for \sdgabbrv{}s.

\begin{theorem}
It can be tested in polynomial time whether a graph
$\sdgsym$ is a \sdgabbrv{}.
\label{thm:recognize:sdgs}
\end{theorem}

\begin{proof}
Verifying the graphical condition of Theorem~\ref{thm:graphcharact}
amounts to testing whether all edges reside within a simplex.
However, knowing that \sdgabbrv{}s are bound graphs,
we can apply an efficient algorithm for bound graph
recognition that
uses radix sort and simplex
elimination and achieves a runtime of $\cO(n+sm)$ 
\citep{Skowronska1984}, where $s \leq n$
is the number of simplexes in the graph. This 
is typically better than $\cO(n^3)$
because large $m$ implies small $s$ and vice versa.
Alternatively, we can apply known fast algorithms to find all
simplicial nodes \citep{kloks00}.
\end{proof}

\section{FINDING FAITHFUL POSETS}

\label{sec:consistency}

We now ask how to find faithful DAGs for simple
marginal independence graphs. We observed that 
marginal independence graphs 
cannot distinguish between transitively equivalent DAGs, so 
a perhaps more natural question is: which \emph{posets} are faithful
to a given graph? 
As pointed out before, we can obtain all DAGs from 
faithful posets in a unique manner by removing transitive edges. 
A further advantage of the poset representation will turn out
to be that the ``smallest'' and ``largest'' faithful
posets can be characterized uniquely (up to isomorphism);
as we shall also see, this is not as easy for DAGs,
except for marginal independence graphs in a certain subclass.

\subsection{Maximal Faithful Posets}

Our first aim is to characterize the ``upper bound'' of the
faithful set. That is, we wish to identify 
those posets for which no edge supergraph is also
faithful. We will show
that a construction described by \citet{Pearl1994}
solves exactly this problem. 

\begin{definition}
For a graph $\sdgsym=(\bV,\bE(\sdgsym))$, the \emph{sink graph}
$\cS(\sdgsym)=(\bV,\bE(\cS(\sdgsym)))$ is constructed as follows: for each edge $u - v$ in $\sdgsym$,
add to $\bE(\cS(\sdgsym))$: (1) an edge $u \to v$ if $\textrm{Bd}(u) \subsetneq \textrm{Bd}(v)$;
(2) an edge $u \gets v$  if $\textrm{Bd}(u) \supsetneq \textrm{Bd}(v)$;
(3) an edge $u - v$  if $\textrm{Bd}(u) = \textrm{Bd}(v)$.
\end{definition}

For instance, the sink graph of the graph in Figure~\ref{fig:marginalbutnotmarkov}a is 
the graph in Figure~\ref{fig:marginalbutnotmarkov}b.

\begin{definition}[\cite{Pearl1994}]
A \emph{sink orientation} of a graph  $\sdgsym$  is any DAG
obtained by replacing every undirected edge of $\cS(\sdgsym)$  
by a directed edge. 
\end{definition}

We first need to state the following.

\begin{lemma}
Every sink orientation of $\sdgsym$ is a poset.
\label{thm:sinkposet}
\end{lemma}

\begin{proof}
Fix a sink orientation $\cG$ and consider any chain $x \to y \to z$.
By construction, this implies that $\textrm{Bd}(x) \subsetneq \textrm{Bd}(z)$.
Hence, if $x$ and $z$ are adjacent in the sink graph, 
then the only possible orientation is $x \to z$. There can be two
reasons why $x$ and $z$  are not adjacent in the 
sink graph: (1) They are not adjacent in $\sdgsym$. But then $\cG$ would not
be faithful, since $\cG$ implies the edge $x - z$. (2) The edge was 
not added to the sink graph. But this contradicts
$\textrm{Bd}(x) \subsetneq \textrm{Bd}(z)$.
\end{proof}

This Lemma allows us to strengthen Theorem~2 by
\citet{Pearl1994} in the sense that we can replace
``DAG'' by ``maximal poset'' (emphasized):

\begin{theorem}
$\cP$ is a \emph{maximal poset} faithful to $\sdgsym$
iff $\cP$  is a sink orientation of $\sdgsym$.
\label{thm:sinkor}
\end{theorem}

The following is also not hard to see.

\begin{lemma}
For a \sdgabbrv{} $\sdgsym$, every DAG $\cG$ that is faithful to $\sdgsym$ is 
a subgraph of some sink orientation of $\sdgsym$.
\label{thm:sinkor2}
\end{lemma}

\begin{proof}
Obviously the skeleton of $\cG$ cannot contain edges that are not
in $\sdgsym$. So, suppose $x \to y$ is an edge in $\cG$ but conflicts
with the sink orientation; that is, the sink graph contains the 
edge $y \to x$. That is the case only if 
$\textrm{Bd}_\sdgsym(y)$ is a proper subset
of $\textrm{Bd}_\sdgsym(x)$. However, in the marginal independence
graph of $\cG$, any node that is adjacent to $x$ (has a
common ancestor) must also
be adjacent to $y$. Thus, the marginal independence graph
of $\cG$ cannot be $\sdgsym$.
\end{proof}

Every maximal faithful poset for $\sdgsym$ can be 
generated by first fixing a topological ordering 
of $\cS(\sdgsym)$ and then generating the DAG that corresponds
to that ordering, an idea that has also been mentioned by~\citet{Drton2008}.
This construction makes it obvious that all maximal faithful
posets are isomorphic.

For curiosity of the reader, 
we  note that $\cS(\sdgsym)$ can also be viewed 
as a \emph{complete partially directed acyclic
graph} (CPDAG), which represents the Markov equivalence
class of edge-maximal DAGs that are faithful with
$\sdgsym$. CPDAGs are used in the context of inferring
DAGs from data \citep{Spirtes2000,Chickering03,Kalisch2007}, 
which is only possible up to Markov 
equivalence.

\subsection{Minimal Faithful Posets}

A minimal faithful poset to $\sdgsym$ is one from which no further relations
can be deleted without entailing more independencies than are given
by $\sdgsym$.

\begin{definition}
Let $\sdgsym=(\bV,\bE)$ be a graph and let $\bI \subseteq \bV$ be an independent
set. Then $\bI_\sdgsym^\to$ is the poset consisting of the nodes in $\bI$, their
neighbors in $\sdgsym$, and directed edges $i \to j$ for each $i,j$ where $j \in N(i)$. 
\end{definition}

For example, Figure~\ref{fig:minposet}b shows the unique $\bI_\sdgsym^\to$ for the graph in
Figure~\ref{fig:minposet}a. 

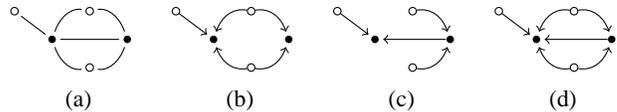
\begin{figure}
\newcommand{\thenodes}{
\tikzstyle{every node}=[circle,fill,inner sep=1pt];
\tikzstyle{every edge}=[bi]
\node [fill=none,draw=black] (a) at (0.5,0.25) {};
\node (b) at (1,0) {};
\node (c) at (2,0) {};
\node [fill=none,draw=black] (d) at (1.5,0.25) {};
\node [fill=none,draw=black] (e) at (1.5,-0.25) {};
}
\begin{subfigure}{.22\columnwidth}
\begin{tikzpicture}[yscale=1.5]
\thenodes
\draw (b) edge (a) edge  (c) edge [bend left=35]  (d) edge [bend right=35]  (e) 
(c) edge  [bend left=35]  (e) edge [bend right=35]  (d);
\end{tikzpicture}
\caption{}
\end{subfigure}\hfill
\begin{subfigure}{.22\columnwidth}
\begin{tikzpicture}[yscale=1.5]
\thenodes
\draw (b) edge [rdir] (a) edge [rdir,bend left=35] (d) edge [rdir,bend right=35] (e) 
(c) edge [rdir,bend left=35] (e) edge [rdir,bend right=35] (d);
\end{tikzpicture}
\caption{}
\end{subfigure}\hfill
\begin{subfigure}{.22\columnwidth}
\begin{tikzpicture}[yscale=1.5]
\thenodes
\draw (b) edge [rdir] (a) (c) edge [rdir,bend left=35] (e) edge [rdir,bend right=35] (d) 
(c) edge [dir] (b);
\end{tikzpicture}
\caption{}
\end{subfigure}\hfill
\begin{subfigure}{.22\columnwidth}
\begin{tikzpicture}[yscale=1.5]
\tikzstyle{every node}=[circle,fill,inner sep=1pt];
\tikzstyle{every edge}=[draw,-]
\thenodes
\draw (b) edge [rdir] (a) edge [rdir,bend left=35] (d) edge [rdir,bend right=35] (e) 
(c) edge [rdir,bend left=35] (e) edge [rdir,bend right=35] (d)
 (c) edge [dir] (b);
\end{tikzpicture}
\caption{}
\end{subfigure}
\caption{(a) A graph $\sdgsym$ with three simplicial nodes $\bI$
(open circles). (b) Its unique minimal faithful poset $\bI^\to_\sdgsym$.
(c,d) The unique faithful DAGs with minimum (c) or maximum (d)
numbers of edges. }
\label{fig:minposet}
\end{figure}

\begin{theorem}
Let $\cU=(\bV,\bE) \in \sdgsym$. Then a poset $\cP$ is
a minimal poset faithful to $\cU$ iff
$\cP = \bI_\cU^\to$ for a set $I$ consisting
of one simplicial vertex for each simplex.
\label{thm:minimalposet}
\end{theorem}

\begin{proof}
We first show that if $\bI$ is a set consisting
of one simplicial node for each simplex, then
$\bI_\sdgsym^\to$ is a minimal faithful poset.
Every edge $e \in \bE(\sdgsym)$ 
resides in a simplex, so it is either adjacent to $\bI$ or both of its
endpoints are adjacent to some $i \in \bI$. In both cases, $\bI_\sdgsym^\to$
implies $e$. Also  $\bI_\sdgsym^\to$ does not imply more edges than are 
in $\sdgsym$. Now, suppose we delete an edge $i \to x$ 
from $\bI_\sdgsym^\to$. This edge must exist in $\sdgsym$, else $i$ was not 
simplicial. But now $\bI_\sdgsym^\to$ no longer implies this edge. Thus,
$\bI_\sdgsym^\to$ is minimal.
Second, assume that $\cP$ is a minimal faithful poset. Assume $\cP$
would contain a sequence of two directed edges $x \to y \to z$.
Then $\cP$ would also contain the edge $x \to z$. But then $y \to z$ 
could be deleted from $\cP$ without changing the dependency graph,
and $\cP$ was not minimal. So, $\cP$ does not contain any directed
path of length more than 1. Next, observe that for each simplex 
in $\cU$, the nodes must all have a common ancestor in $\cP$.
Without paths of length $>1$, this 
is only possible if one node $i$ in the simplex is a parent of all other 
nodes, and there are no edges among the child nodes of $i$.
Finally, each such $i$ must be a simplicial node in
$\sdgsym$; otherwise, it
would reside in two or more simplexes, and would have to be 
the unique parent in those simplexes. But then the children of $i$ 
would form a single simplex in $\sdgsym$.
\end{proof}

Like the maximal posets, all minimal posets are thus isomorphic.
We point out that the minimal posets contain no transitive
edges and therefore, they are also edge-minimal faithful DAGs.
However, this does not imply that minimal posets
have the smallest possible number
of edges amongst all faithful DAGs (Figure~\ref{fig:minposet}).
There appears to be no straightforward characterization of the 
DAGs with the smallest number of edges for marginal independence graphs
in general. However, a beautiful one exists for the 
subclass of trivially perfect graphs.

\begin{definition}
A \emph{tree poset} is a poset whose transitive reduction is a tree
(with edges pointing towards the root).
\end{definition}

\begin{theorem}
A connected \sdgabbrv{} $\sdgsym$ has a faithful tree poset iff it is
trivially perfect.
\end{theorem}

\begin{proof}
The bound graph of a tree poset is 
identical to its \emph{comparability graph}
\citep{Brandstaedt1999}, which is the 
skeleton of the poset. Comparability graphs of tree posets
coincide with trivially perfect graphs \citep{Wolk1965}.
\end{proof}

Since no connected graph on $n$ nodes can have fewer 
edges than the transitive reduction of a tree poset on the 
same nodes (i.e., $n-1$), tree posets coincide with 
faithful DAGs having the smallest possible number of edges.

How do we construct a tree for a given trivially perfect graph? 
Every such graph must have a \emph{central point}, which is
a node that is adjacent to all other nodes. We set this node as 
the sink of the tree, and continue recursively with the subgraphs
obtained after removing the central point. Each subgraph is 
also trivially perfect and can thus be oriented into a tree.
After we are done, we link the sinks of the trees of the subgraphs
to the original central point to obtain the full tree
\citep{Wolk1965}.


\section{FINDING FAITHFUL DAGS}

\label{sec:algo}

If a given marginal independence graph $\sdgsym$ admits
faithful DAG models,
then it is of interest to enumerate these.
A trivial enumeration procedure 
is the following: start with the
sink graph of $\sdgsym$, choose an arbitrary edge $e$, and form
all 2 or 3 subgraphs obtained by keeping $e$ (if it 
is directed), orienting $e$ (if it is undirected), or deleting
it. Apply the procedure recursively to these subgraphs.
During the recursion, do not touch edges that have been
previously chosen. If the current graph is a DAG that
is faithful to $\sdgsym$, output it; otherwise, 
stop the recursion.

However, we can do better by exploiting the results
of the previous section, which will allow us to derive
enumeration algorithms that generate representations
of multiple DAGs at each step.

\subsection{Enumeration of Faithful DAGs}

Having characterized the maximal and minimal faithful 
posets, we are now ready to
construct an enumeration procedure for all DAGs that
are faithful to a given graph.
We first state the following combination of
Theorem~\ref{thm:sinkor} and Theorem~\ref{thm:minimalposet}. 

\begin{proposition}
A DAG $\cG=(\bV,\bE(\cG))$ is faithful to a \sdgabbrv{} $\sdgsym=(\bV,\bE(\sdgsym))$ 
iff  (1) $\cG$ is an edge subgraph of some sink orientation of
$\sdgsym$ and (2) the transitive closure of $\cG$ is an 
edge supergraph of $\bI_\sdgsym^\to$ for some node
set $I$ consisting of one simplicial node for each simplex.
\label{prop:dagcharact}
\end{proposition}

From this observation, we can derive our first construction
procedure for faithful DAGs.

\begin{proposition}
A DAG $\cG$ is faithful to a \sdgabbrv{} $\sdgsym=(\bV,\bE(\sdgsym))$ iff
it can be generated by the following steps.
(1) Pick any set $\bI \subseteq \bV$
consisting of one simplicial node for each simplex.
(2) Generate any DAG on the nodes $\bV\setminus\bI$
that is an edge subgraph of some sink orientation of $\sdgsym$. 
(3) Add any subset of edges from $\bI_\sdgsym^\to$ such
that the transitive closure of the resulting graph 
contains all edges of $\bI_\sdgsym^\to$.
\label{prop:dagalgorithm}
\end{proposition}

While step (3) may seem ambiguous, 
Figure~\ref{fig:dagalgorithm} illustrates 
that 
after step (2), the edges from 
$\bI_\sdgsym^\to$ decompose nicely into \emph{mandatory}
and \emph{optional} ones. This means that we can
in fact stop the construction procedure after step (2)
and output a ``graph pattern'', in which some edges 
are marked as optional. This is helpful in light
of the potentially huge space of faithful models,
because every graph pattern can represent an exponential 
number  of DAGs.

\begin{figure}

\tikzstyle{every node}=[circle,fill,inner sep=1pt]

\tikzstyle{every edge}=[bi]
\hfill
\begin{subfigure}{.16\columnwidth}
\centering
\begin{tikzpicture}[rotate=-90,xscale=1.2,yscale=.5]
\node (a) at (0.5,0) {};
\node (b) at (1,0) {};
\node (c) at (1.5,0) {};
\node (d) at (1,1) {};
\node (e) at (1,-1) {};
\draw (e) edge [bi,bend left=90]  (d);
\draw (a) edge (b) edge (d) edge (e)
(c) edge (b) edge (d) edge (e)
(b) edge (d) edge (e);
\end{tikzpicture}%
\caption{}
\end{subfigure}
\newcommand{\bild}[1]{%
\begin{tikzpicture}[rotate=-90,xscale=1.2,yscale=.5]
\tikzstyle{every edge}=[dir]
\node [fill=none,draw=black] (a) at (0.5,0) {};
\node (b) at (1,0) {};
\node [fill=none,draw=black] (c) at (1.5,0) {};
\node (d) at (1,1) {};
\node (e) at (1,-1) {};
\draw [white] (e) edge [bi,bend left=90]  (d);
\draw #1;
\end{tikzpicture}%
}
\hfill
\begin{subfigure}{.64\columnwidth}

\centering
\bild{
(a) edge (b) edge (d) edge (e)
(c) edge (b) edge (d) edge (e)
} \hfill
\bild{
(a) edge [dashed] (b) edge (d) edge (e)
(c) edge [dashed]  (b) edge (d) edge (e)
(d) edge [thick] (b)
} \hfill
\bild{
(a) edge (b) edge (d) edge [dashed]  (e)
(c) edge (b) edge (d) edge [dashed]  (e)
(b) edge [thick] (e)
} \hfill
\bild{
(a) edge [dashed]   (b) edge (d) edge [dashed]  (e)
(c) edge [dashed]   (b) edge (d) edge [dashed]  (e)
(d) edge [thick] (b)
(b) edge [thick] (e)
}

\medskip

\bild{
(a) edge (b) edge (d) edge [dashed] (e)
(c) edge (b) edge (d) edge [dashed](e)
(d) edge [thick,bend right=90]  (e)
} \hfill
\bild{
(a) edge [dashed] (b) edge (d) edge[dashed] (e)
(c) edge [dashed] (b) edge (d) edge [dashed](e)
(d) edge [thick,bend right=90]  (e)
(d) edge [thick] (b)} \hfill
\bild{
(a) edge (b) edge (d) edge [dashed](e)
(c) edge (b) edge (d) edge [dashed](e)
(d) edge [thick,bend right=90]  (e)
(b) edge [thick] (e)
} \hfill
\bild{
(a) edge [dashed] (b) edge (d) edge [dashed](e)
(c) edge [dashed] (b) edge (d) edge [dashed](e)
(d) edge [thick,bend right=90]  (e)
(d) edge [thick] (b)
(b) edge [thick] (e)
}
\caption{}
\end{subfigure}
\hfill \null

\caption{Example of the procedure in 
Proposition~\ref{prop:dagalgorithm} that, given a
\sdgabbrv{} (a), enumerates all faithful DAGs (b).
For brevity,
only the graphs that correspond to a fixed
topological ordering are displayed.  Only
one set $\bI$ (open circles) can be
chosen in step (1). Thick edges and filled nodes 
highlight the DAG $\cG$.
Mandatory edges (solid) link $\bI$ to the sources of $\cG$; 
if any such edge was absent, one of the relationships in 
the poset $\bI_\sdgsym^\to$ would be missing.
Optional edges (dashed) are transitively implied
from the mandatory ones and $\cG$.
}
\label{fig:dagalgorithm}
\end{figure}
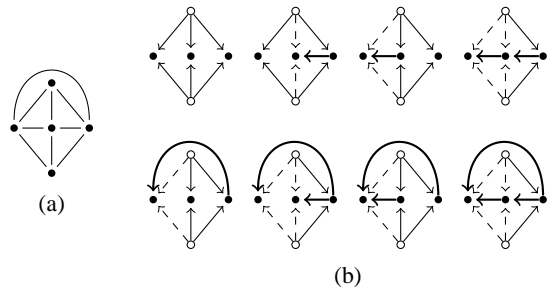

\subsection{Enumeration of Faithful Posets}

The DAGs resulting from the procedure in
Proposition~\ref{prop:dagalgorithm} are in general
redundant because no care is taken to avoid generating
transitive edges. By combining Propositions~\ref{prop:dagcharact}
and~\ref{prop:dagalgorithm},
we obtain an algorithm that
generates sparse, non-redundant representations of 
the faithful DAGs.

\begin{theorem}
A poset $\cP$ is faithful to $\sdgsym=(\bV,\bE(\sdgsym))$ iff
it can be generated by the following steps.
(1) Pick any set $\bI \subseteq \bV$
consisting of one simplicial node for each simplex.
(2) Generate a poset $\cP$ on the nodes $\bV\setminus\bI$
that is an edge subgraph of some sink orientation of $\sdgsym$. 
(3) Add $\bI_\sdgsym^\to$ to $\cP$.
\label{th:posetenum}
\end{theorem}

A nice feature of this construction is that step (3) is unambiguous:
every choice for $\bI$ in step (1) and $\cP$ in step
(2) yields exactly one poset. Figure~\ref{algo:enumsep} gives an
explicit pseudocode for an algorithm that uses Theorem~\ref{th:posetenum}
to enumerate all faithful posets.

\begin{figure}[ht]
\begin{algorithmic}
\Function{FaithfulPosets}{$\sdgsym=(\bV(\sdgsym),\bE(\sdgsym))$}
\Function{ListPosets}{$\cG,\cS,\bR,\bI^\to_\sdgsym$} 
    \If{ $\cG$ is acyclic and atransitive}
    \State{Output $\cG \cup \bI^\to_\sdgsym$}
    \If{ skeleton of $\cG \subsetneq $ skeleton of $\cS$  }
    \State{$e \gets $ some edge consistent with $\bE(\cS) \setminus \bR$}
    \State{{\sc  ListPosets}($\cG,\cS,\bR\cup\{e\},\bI^\to_\sdgsym$)}
    \State{$\bE(\cG) \gets \bE(\cG) \cup \{e\}$}
    \State{{\sc  ListPosets}($\cG,\cS,\bR\cup\{e\},\bI^\to_\sdgsym$)}
  \EndIf

  \EndIf
\EndFunction
\For{all node sets $\bI$ of $\sdgsym$ consisting of 
one simplicial \\ \hspace{1.5cm}  node per simplex}
\State{$\cG \gets $ empty graph on nodes of $\bV(\sdgsym)\setminus\bI$} 
\State{$\cS \gets $ sink graph of $\sdgsym$ on nodes of  $\bV(\sdgsym)\setminus\bI$}
\State{{\sc  ListPosets}($\cG,\cS,\emptyset,\bI^\to_\sdgsym$)}
\EndFor
\EndFunction
\end{algorithmic}
\caption{Enumeration algorithm for faithful posets.}
\label{algo:enumsep}
\end{figure}

Our algorithm is efficient in the sense that at every internal
node in its recursion tree, it outputs a faithful poset.
At every node we need to 
evaluate whether the current $\cG$ is acyclic and atransitive
(i.e., contains no transitive edges), which can be done in polynomial
time. Also simplexes and their simplicial vertices can be found 
in polynomial time \cite{kloks00}.
Thus, our algorithm is a \emph{polynomial delay enumeration
algorithm} similar to the ones used to enumerate 
adjustment sets for DAGs \citep{textor11_uai,vanconstructing}. 
Figure~\ref{fig:consistentposets} shows an example output
for this algorithm.

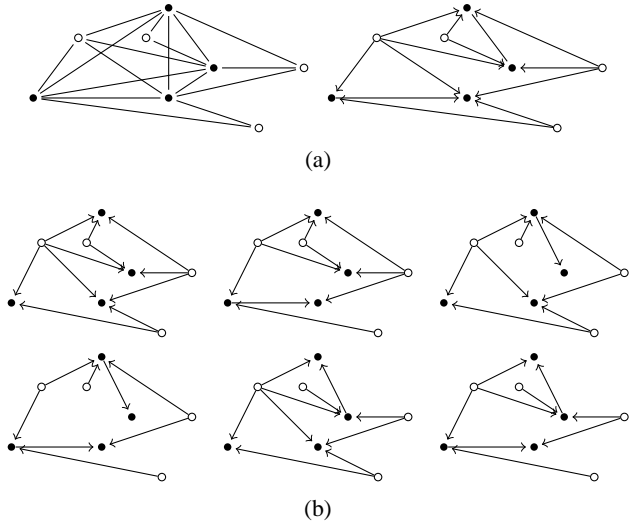
\begin{figure}

\newcommand\knoten{}
\tikzstyle{every node}=[circle,fill,inner sep=1pt]
\tikzstyle{simplicial}=[draw=black,fill=white]%

\renewcommand{\knoten}{
\node [simplicial] (v1) at (1.5,1.5) {};
\node [simplicial] (v2) at (3.5,0) {};
\node [simplicial] (v3) at (4,1) {};
\node [simplicial] (v4) at (2.25,1.5) {};
\node (v5) at (1,0.5) {};
\node (v8) at (2.5,.5) {};
\node (v6) at (3,1) {};
\node (v7) at (2.5,2) {};
}%
\begin{subfigure}{\columnwidth}
\centering
\begin{tikzpicture}[xscale=1.2,yscale=.8]
\knoten
\tikzstyle{every edge}=[bi]
\draw (v1) edge  (v5) 
(v2) edge  (v5) 
(v1) edge  (v6) 
(v3) edge  (v6) 
(v4) edge  (v6) 
(v5) edge  (v6) 
(v1) edge  (v7) 
(v3) edge  (v7) 
(v4) edge  (v7) 
(v5) edge  (v7) 
(v6) edge  (v7) 
(v1) edge  (v8) 
(v2) edge  (v8) 
(v3) edge  (v8) 
(v5) edge  (v8) 
(v6) edge  (v8) 
(v7) edge  (v8);
\end{tikzpicture} \ \ 
\begin{tikzpicture}[xscale=1.2,yscale=.8]
\knoten
\tikzstyle{every edge}=[dir]
\draw (v1) edge  (v5) 
(v2) edge  (v5) 
(v1) edge  (v6) 
(v3) edge  (v6) 
(v4) edge  (v6) 
(v7) edge [bi] (v6) 
(v1) edge  (v7) 
(v3) edge  (v7) 
(v4) edge  (v7) 
(v1) edge  (v8) 
(v2) edge  (v8) 
(v3) edge  (v8) 
(v5) edge  (v8);
\end{tikzpicture}
\caption{}
\end{subfigure}

\bigskip

\begin{subfigure}{\columnwidth}

\centering

\begin{tikzpicture}[scale=.8]
\knoten
\tikzstyle{every edge}=[dir]
\draw (v1) edge  (v5) 
(v2) edge  (v5) 
(v1) edge  (v6) 
(v3) edge  (v6) 
(v4) edge  (v6) 
(v1) edge  (v7) 
(v3) edge  (v7) 
(v4) edge  (v7) 
(v1) edge  (v8) 
(v2) edge  (v8) 
(v3) edge  (v8) 
;
\end{tikzpicture}\hfill
\begin{tikzpicture}[scale=.8]
\knoten
\tikzstyle{every edge}=[dir]
\draw (v1) edge  (v5) 
(v2) edge  (v5) 
(v1) edge  (v6) 
(v3) edge  (v6) 
(v4) edge  (v6) 
(v1) edge  (v7) 
(v3) edge  (v7) 
(v4) edge  (v7) 
(v3) edge  (v8) 
(v5) edge  (v8)
;
\end{tikzpicture}\hfill
\begin{tikzpicture}[scale=.8]
\knoten
\tikzstyle{every edge}=[dir]
\draw (v1) edge  (v5) 
(v2) edge  (v5) 
(v7) edge  (v6) 
(v1) edge  (v7) 
(v3) edge  (v7) 
(v4) edge  (v7) 
(v1) edge  (v8) 
(v2) edge  (v8) 
(v3) edge  (v8) 
;
\end{tikzpicture}\\[.5em]
\begin{tikzpicture}[scale=.8]
\knoten
\tikzstyle{every edge}=[dir]
\draw (v1) edge  (v5) 
(v2) edge  (v5) 
(v7) edge  (v6) 
(v1) edge  (v7) 
(v3) edge  (v7) 
(v4) edge  (v7) 
(v3) edge  (v8) 
(v5) edge  (v8)
;
\end{tikzpicture}\hfill
\begin{tikzpicture}[scale=.8]
\knoten
\tikzstyle{every edge}=[dir]
\draw (v1) edge  (v5) 
(v2) edge  (v5) 
(v1) edge  (v6) 
(v3) edge  (v6) 
(v4) edge  (v6) 
(v7) edge [rdir] (v6) 
(v1) edge  (v7) 
(v1) edge  (v8) 
(v2) edge  (v8) 
(v3) edge  (v8) 
;
\end{tikzpicture}\hfill
\begin{tikzpicture}[scale=.8]
\knoten
\tikzstyle{every edge}=[dir]
\draw (v1) edge  (v5) 
(v2) edge  (v5) 
(v1) edge  (v6) 
(v3) edge  (v6) 
(v4) edge  (v6) 
(v7) edge [rdir] (v6) 
(v1) edge  (v7) 
(v3) edge  (v8) 
(v5) edge  (v8)
;
\end{tikzpicture}
\caption{}
\end{subfigure}

\caption{(a) A graph $\sdgsym$ and its sink graph. 
(b) Transitive reductions of all 6 faithful posets 
that are 
generated by Algorithm~{\sc FaithfulPosets} for 
the input graph (a).
}
\label{fig:consistentposets}
\end{figure}

\section{EXAMPLE APPLICATIONS}

\label{sec:combinatorics}

In this section, we apply the previous results to explore some explicit
combinatorial properties of \sdgabbrv{}s and their faithful DAGs.

\subsection{Counting \sdgabbrv{}s}

We revisit the question: when can a marginal independence
graph allow a causal interpretation \citep{Pearl1994}? 
More precisely, we ask \emph{how many} marginal independence graphs on $n$
variables are \sdgabbrv{}s. 
We reformulate this question into a version that has
been investigated in the context of poset theory. 
Let the \emph{height} of a poset $\cP$ be the length of a longest
path in $\cP$.  The following is an obvious 
implication of Theorem~\ref{thm:minimalposet}.

\begin{corollary}
The number $M(n)$ of non-isomorphic \sdgabbrv{}s with $n$ nodes is equal
to the number of non-isomorphic posets on $n$ variables of height 1.
\end{corollary}

Enumeration of posets is a highly nontrivial problem,
and an intensively studied one. The online encyclopedia
of integer sequences (OEIS) tabulates $M(n)$ for $n$ up to 40
\citep{oeisA007776}.
We give the first 10 entries of the sequence in Table~\ref{tab:numbers}
and compare it to the number of graphs in general
(up to isomorphism). As we observe, the fraction of 
graphs that admit a DAG on the same variables
decreases swiftly as $n$ increases.

\begin{table}
\centering
\begin{tabular}{rrrr}
 & connected & conn. & unique \\
$n$ & graphs & \sdgabbrv{}s & DAG \\
2 & 1 & 1 & 0 \\
3 & 2 & 2 & 1 \\
4 & 6 & 4 & 1 \\
5 & 21 & 10 & 2 \\
6 & 112 & 27 & 4 \\
7 & 853 & 88 & 10 \\
8 & 11,117 & 328 & 27 \\
9 & 261,080 & 1,460 & 90 \\
10  & 11,716,571 & 7,799 & 366 \\
\end{tabular}
\caption{
Comparison of the number of unlabeled connected graphs
with $n$ nodes to the number of such graphs that
are also \sdgabbrv{}s.
For $n=13$ (not shown), 
non-\sdgabbrv{}s outnumber \sdgabbrv{}s by more than $10^7:1$.
}
\label{tab:numbers}
\end{table}

\subsection{Graphs with a Unique Faithful DAG}

From a causal inference viewpoint, the best we can hope for
is a \sdgabbrv{} to which only single, unique DAG
is faithful.
The classical example is the 
graph $\cdot-\cdot-\cdot$, which for more
than 3 nodes generalizes
to a ``star'' graph. However, for 5 or more
nodes there are graphs other than the star
which also induce a single unique DAG.
Combining Lemma~\ref{thm:sinkor2} and
Theorem~\ref{thm:minimalposet} allows
for a simple characterization of all such \sdgabbrv{}s.

\begin{corollary}
A \sdgabbrv{} $\sdgsym$ with $n$ nodes 
has a unique faithful DAG iff each of its simplexes
contains only one simplicial node and its sink
orientation equals $\bI_\sdgsym^\to$.
\end{corollary}

Based on this characterization, we computed the number
of \sdgabbrv{}s with unique DAGs for $n$ up till $9$ (Table~\ref{tab:numbers}). 
Interestingly, this integer sequence does not seem to correspond to any
known one.

\subsection{Information Content of a \sdgabbrv{}}

How much information does a marginal independence graph contain? 
Let us denote the number of posets on $n$ variables by 
$P(n)$. After observing a marginal independence graph
$\sdgsym$, the number of models that are still faithful
to the data reduces to size $P(n) - k(\sdgsym)$, where
$k(\sdgsym) \leq P(n)$ 
(indeed, quite often $k(\sdgsym) = P(n)$ 
as we can see in Table~\ref{tab:numbers}).
Of course, the number 
$k(\sdgsym)$ strongly depends on the structure
of the \sdgabbrv{} $\sdgsym$. But even in the worst case
when $\sdgsym$ is a complete graph,
the space of possible models is still reduced because not all
DAGs entail a complete marginal independence graph.

Thus, the following simple consequence of
Theorem~\ref{thm:minimalposet} helps to derive 
a worst-case bound on how much a \sdgabbrv{} reduces structural
uncertainty with respect to the model space of posets with
$n$ variables.

\begin{corollary}
The number of faithful posets with respect to
a complete graph with $n$ nodes is $n$ times the number of posets 
with $n-1$ nodes.
\end{corollary}

Table~\ref{tab:numbers2} lists the number of possible
posets before and after observing a complete \sdgabbrv{}
for up to 10 variables. In this sense, at $n=10$, the
uncertainty is reduced about 15-fold.

\begin{table}
\centering
\begin{tabular}{rrr}
$n$ & posets with $n$ nodes & faithful to $C_n$ \\
1 & 1 & 1 \\
2 & 3 & 2 \\
3 & 19 & 9 \\
4 & 219 & 76 \\
5 & 4,231 & 1,095 \\
6 & 130,023 & 25,386 \\
7 & 6,129,859 & 910,161 \\
8 & 431,723,379 & 49,038,872 \\
9  & 44,511,042,511 & 3,885,510,411 \\
10 &  6,611,065,248,783 & 445,110,425,110 \\
\end{tabular}
\caption{
Possible labelled posets on $n$ variables 
before and after observing a complete SMIG $C_n$.
}
\label{tab:numbers2}
\end{table}

We note that a similar but more technical
analysis is possible for uncertainty reduction with respect
to DAGs instead of posets. We omit this due to space limitations.

\section{MODELS WITH LATENT VARIABLES}

\label{sec:latents}

In this section we consider situations in which 
a graph $\sdgsym$ is not a \sdgabbrv{} (which
can be detected using the algorithm in 
Theorem~\ref{thm:recognize:sdgs}).
Similarly to the definition proposed in  
\citet{pearl1987logic} for the general 
dependency models, to obtain faithful DAGs for
such graphs we will extend the DAGs
with some auxiliary nodes. We generalize 
Definition~\ref{definition:dependency:graph}
as follows.

\begin{definition}
Let $\sdgsym=(\bV,E(\sdgsym))$ be a  graph
and let $\bQ$, with $\bQ\cap\bV=\emptyset$, be a set of 
auxiliary nodes. A DAG $\cG=(\bV\cup \bQ, E(\cG))$ 
is faithful to $\sdgsym$ if for all 
$v,w\in \bV$, $v-w \in E(\sdgsym)$ iff
$v$ and $w$ have a common ancestor in $\cG$. 
\end{definition}

The result below follows immediately from 
Proposition~\ref{proposition:induced:subgraph:MDG}.
\begin{proposition}
For every  graph $\sdgsym$ there exists 
a faithful DAG $\sdgsym$ with some auxiliary nodes.
\end{proposition}

Obviously, if $\sdgsym \in \sdgset$ then there exists a
faithful DAG to $\sdgsym$ with $\bQ=\emptyset$.
For $\sdgsym\notin\sdgset$, from the proof of 
Proposition~\ref{proposition:induced:subgraph:MDG}
it follows that there exists a set $\bQ$ of at most 
$|E(\sdgsym)|$ nodes and a DAG $\cG$ 
such that $\cG$ is faithful to $\sdgsym$
with auxiliary nodes $\bQ$.
But the problem arises to minimize 
the cardinality of~$\bQ$.

\begin{theorem}
The problem to decide if for a given graph $\sdgsym$ and 
an integer $k$, there exists a faithful DAG 
with at most $k$ auxiliary nodes,  is NP-complete. 
\end{theorem}
\begin{proof}
It is easy to see that the problem is in NP. 
To prove that it is NP-hard, we show a polynomial time 
reduction from the edge clique cover problem, that is known
to be NP-complete~\citep{karp1972reducibility}.
Recall that the problem  edge clique cover
is to decide if for a graph $\sdgsym$ 
and an integer $k$ there exist a set of $k$ 
subgraphs of $\sdgsym$, such that each subgraph 
is a clique and each edge of $\sdgsym$ is contained 
in at least one of these subgraphs?

Let $\sdgsym=(\bV,\bE)$ and $k$ be an instance
of the edge clique cover problem, with 
$\bV=\{v_1,\ldots, v_n\}$. 
We construct the marginal independence 
graph $\sdgsym'$ as follows. Let $\bW=\{w_1,\ldots, w_n\}$.
Then $V(\sdgsym')=\bV\cup \bW$ and 
$E(\sdgsym')=\bE\cup\{v_i-w_i: i=1,\ldots,n\}$.
Obviously, $\sdgsym'$ can be constructed from $\sdgsym$
in polynomial time.
We claim that $\sdgsym=(\bV,\bE)$ can be 
covered by $\le k$ cliques iff 
for $\sdgsym'$ there exists a faithful 
DAG $\cG$ with at most $k$ auxiliary nodes.

Assume first that  $\sdgsym=(\bV,\bE)$ can be 
covered by at most $k$ cliques, let us say  
$C_1,\ldots, C_{k'}$, with $k'\le k$.
Then we can construct a faithful 
DAG $\cG$ for $\sdgsym'$ with $k'$ auxiliary nodes
as follows. Its set of nodes is $V(\cG)=\bV\cup \bW \cup \bQ$,
where $\bQ=\{q_1,\ldots,q_{k'}\}$. The edges 
$E(\cG)$ can be defined as
\[
  \{w_i \to v_i :  i=1,\ldots,n\} \cup \bigcup_j \{q_j \to v : v\in C_j \}.
\] 
It is easy to see that $\cG$ is faithful to $\sdgsym'$.

Now assume that a DAG $\cG$, with at most 
$k$ auxiliary nodes $\bQ$, is faithful to $\sdgsym'$. 
From the construction of $\sdgsym'$ it follows that 
for all different nodes $v_i,v_j\in \bV$ there is 
no directed path from $v_i$ to $v_j$ in $\cG$.
If such a path exists, then $v_i$ is an ancestor 
of $v_j$  in $\cG$.
Since $v_i-w_i$ is an edge of $\sdgsym'$,
the nodes $v_i$ and $w_i$ have a common 
ancestor in $\cG$, which must be also a
common ancestor of $w_i$ and $v_j$ -- 
a contradiction because $w_i$ and $v_j$ 
are not incident in $\sdgsym'$.
Thus, all treks connecting pairs of nodes from $\bV$
in $\cG$ must contain auxiliary nodes.
 
Next, we slightly modify $\cG$: for each $w_i$
we remove all incident edges and add the new  edge 
$w_i \to v_i$. The resulting graph $\cG'$, 
is a DAG which remains faithful to $\sdgsym'$. 
Indeed, we cannot obtain a directed cycle in the 
$\cG'$ since no $w_i$ has an in-edge
and the original $\cG$ was a DAG.
To see that the obtained DAG remains faithful 
to $\sdgsym'$ note first that after the modifications, 
$w_i$ and $v_i$ have a common ancestor in $\cG$ 
whereas  $w_i$ and $v_j$, with $i\neq j$, do not.
Otherwise, it would imply a directed path from 
$v_i$ to $v_j$ since $w_i$ is the only possible 
ancestor of both nodes -- a contradiction.
Finally, note that any trek connecting $v_i$ 
and $v_j$ in $\cG$ cannot contain a node from $\bW$.
Similarly, no trek between $v_i$ and $v_j$ in 
$\cG'$ contains a node from $\bW$.
We get that $v_i$ and $v_j$ have a common ancestor in $\cG$
iff they have a common ancestor in $\cG'$.

Thus, in $\cG'$ the auxiliary 
nodes $\bQ$ are incident to $\bV$, but 
not to nodes from $\bW$. 
Below we modify $\cG'$ further and obtain a 
DAG $\cG''$, in which every auxiliary node is 
incident with a node in $V$ via an out-edge only.
To this aim we remove from $\cG'$ all edges going 
out from a node in $V$ to a node in~$Q$.

Obviously, if $v_i$ and $v_j$ have a common ancestor
in $\cG''$, then they also have a common ancestor in $\cG'$,
because $E(\cG'')\subseteq E(\cG')$.
The opposite direction  
follows from the fact we have shown at 
the beginning of this proof that
for all different nodes $v_i,v_j\in \bV$ there is 
no directed path from $v_i$ to $v_j$ in $\cG$.
This is true also for $\cG'$. 
Thus, if $v_i$ and $v_j$ have a common ancestor,
say $x$,  in $\cG'$ then $x\in \bQ$ and 
there exist directed paths $x\to y_1 \to \ldots y_{r}\to v_i$
and $x\to y'_1 \to \ldots y'_{r'}\to v_j$ such that 
also all $y_1, \ldots ,y_{r}$ and $y'_1, \ldots ,y'_{r'}$
belong to $\bQ$. But from the construction of $\cG''$
it follows that both paths belong also to $\cG''$.

Since $\cG''$ is faithful to $\sdgsym$,
for every auxiliary node $Q$ the subgraph 
induced by its children 
$\textit{Ch}(Q)\cap \bV$ in $\cG''$ 
is a clique in $\sdgsym'$.
Moreover every edge $v_i-v_j$
of the graph $\sdgsym$ belongs to at least one such clique. 
Thus the subgraphs induced by 
$\textit{Ch}(q_1)\cap \bV, \ldots, \textit{Ch}(q_{k'})\cap \bV$,
with $k' \le k$,
are cliques that cover $\sdgsym$.
\end{proof}

\section{DISCUSSION}

Given a graph that represents a set of pairwise
marginal independencies, which causal structures on the same 
variables might have generated this graph? Here we 
characterized all these
structures, or alternatively, all maximal and minimal ones.
Furthermore, we have shown that it is possible to deduce
how many exogenous variables (which correspond to 
simplicial nodes) the causal structure might have,
and even to tell whether it might be a tree. For graphs that do 
not admit a DAG on the same variables, we have
studied the problem of explaining the data with as few 
additional variables as possible, and proved it to be NP-hard.
This may be surprising; the related problem of
finding a mixed graph that is Markov equivalent to a bidirected
graph and has as few bidirected edges as possible
is efficiently solvable \citep{Drton2008}.

The connection to posets 
emphasizes that
sets of faithful DAGs have complex combinatorics.
Indeed, if there are no pairwise independent 
variables, then we obtain the 
classical poset enumeration problem
\citep{Brinkmann2002}. Our current, unoptimized implementation of 
the algorithm in Figure~\ref{algo:enumsep} allows us to deal with 
dense graphs up to about 12 nodes (sparse graphs are easier
to deal with).
We point out that our enumeration algorithms operate 
with a ``template graph'', i.e., the sink orientation. It is
possible to incorporate certain kinds of background knowledge, 
like a time-ordering of the variables,
into this template graph by deleting some edges. Such further
constraints could greatly reduce the search space. Another
additional constraint that could be used for linear models is the
precision matrix \citep{Cox1993,Pearl1994}, though 
finding DAGs that explain a given precision
matrix is NP-hard in general \citep{Verma1993},

We observed that the pairwise
marginal independencies substantially reduce
 structural uncertainty even in the worst case 
(Table~\ref{tab:numbers}). Causal inference
algorithms could exploit this to reduce
the number of CI tests. The PC 
algorithm  \citep{Kalisch2007}, for instance, forms the marginal 
independence graph as a first stage before
performing any CI tests. At that stage, it could
be immediately tested if the resulting graph is a 
\sdgabbrv{}, and if not, the algorithm can terminate
as no faithful DAG exists. 

In summary, we have mapped out the space of causal
structures that are faithful to a given set of pairwise
marginal independencies using 
constructive criteria that lead to 
well-structured enumeration procedures. The central
idea underlying our results is that faithful models
for marginal independencies are better described by posets
than by DAGs.
Our results allow
to quantify how much our uncertainty about a causal
structure is reduced when we invoke the faithfulness
assumption and 
observe a set of marginal independencies.

It future work, it would be interesting to extend our approach to
small (instead of empty) conditioning sets, which would
cover cases where we only wish to perform 
CI tests with low dimensionality. 


\clearpage

\bibliographystyle{abbrvnat}

\clearpage

\appendix

\end{document}